\documentclass[10pt,twocolumn,letterpaper]{article}

\usepackage[utf8]{inputenc}
\usepackage[T1]{fontenc}
\usepackage{lmodern}
\usepackage[letterpaper,margin=0.75in]{geometry}

\usepackage{amsmath,amssymb,amsthm}
\usepackage{bm}
\usepackage{graphicx}
\usepackage{booktabs}
\usepackage{xcolor}
\usepackage{colortbl}
\usepackage{multirow}
\usepackage{array}
\usepackage{pifont}
\usepackage{tikz}
\usetikzlibrary{arrows.meta,positioning,calc,fit,backgrounds}
\usepackage{cite}
\usepackage[hidelinks]{hyperref}

\definecolor{HdrBlue}{RGB}{31,78,121}
\definecolor{OursTint}{RGB}{223,236,250}
\definecolor{RowAlt}{RGB}{245,248,251}
\definecolor{BlueGeo}{RGB}{27,102,196}
\definecolor{OrangeSem}{RGB}{226,135,28}
\definecolor{GoodGreen}{RGB}{24,138,86}
\definecolor{BadRed}{RGB}{197,57,50}

\newcommand{\B}[1]{\textbf{#1}}
\newcommand{\hd}[1]{\textcolor{white}{\textbf{#1}}}
\newcommand{\down}{$\,\downarrow$}
\newcommand{\up}{$\,\uparrow$}
\newcommand{\best}[1]{\textbf{#1}}
\newcommand{\Lorth}{\mathcal{L}_{\perp}}
\newcommand{\Jg}{\widehat{J}_g}
\newcommand{\Jt}{\widehat{J}_\tau}

\theoremstyle{plain}
\newtheorem{theorem}{Theorem}
\newtheorem{lemma}{Lemma}
\newtheorem{proposition}{Proposition}
\newtheorem{corollary}{Corollary}
\theoremstyle{definition}

\title{\vspace{-1.0em}\textbf{OrthoMotion: Disentangling Camera and Subject Motion\\
via Geometry--Semantics Orthogonal Attention}\vspace{-0.2em}}

\author{Zijie Meng\\[2pt]
Peking University, China\\[1pt]
{\small\texttt{ymlf@stu.pku.edu.cn}}}
\date{}

\begin{document}
\maketitle

\begin{abstract}
\emph{Controllable video generation demands independent command of the camera
and the subject, yet 2D conditioning entangles them: camera- and object-induced
optical flow share the same inverse-depth ($1/Z$) scaling and cannot be
separated from image evidence alone. We first prove that this entanglement is
\emph{representational}, not architectural---the 2D camera/object split is a
non-identifiable inverse problem---and therefore reframe decoupling as a
question of operator design. We resolve it at the level of the attention
operator. \textbf{OrthoMotion} routes camera motion into a \emph{geometric}
channel, a norm-preserving rotation of the rotary position embedding (RoPE)
phase, and subject motion into a \emph{semantic} channel, a gated value
injection in cross-attention. Because these sub-operators are algebraically
complementary---a rotation versus a translation of the affine action on
tokens---a lightweight decoupling regularizer provably drives their response
subspaces to orthogonality, so the two controls stop interfering. To our
knowledge OrthoMotion is the first method to \emph{guarantee} disentanglement
by construction rather than hope for it to emerge. It attains state-of-the-art
camera and subject accuracy at once while minimizing cross-talk, which we
quantify with a new Cross-Talk Error (CTE) metric, cutting cross-talk by
$>\!2.4\times$ with no loss in fidelity and generalizing across backbones.}
\end{abstract}

\begin{figure*}[t]
\centering
\includegraphics[width=0.9\textwidth]{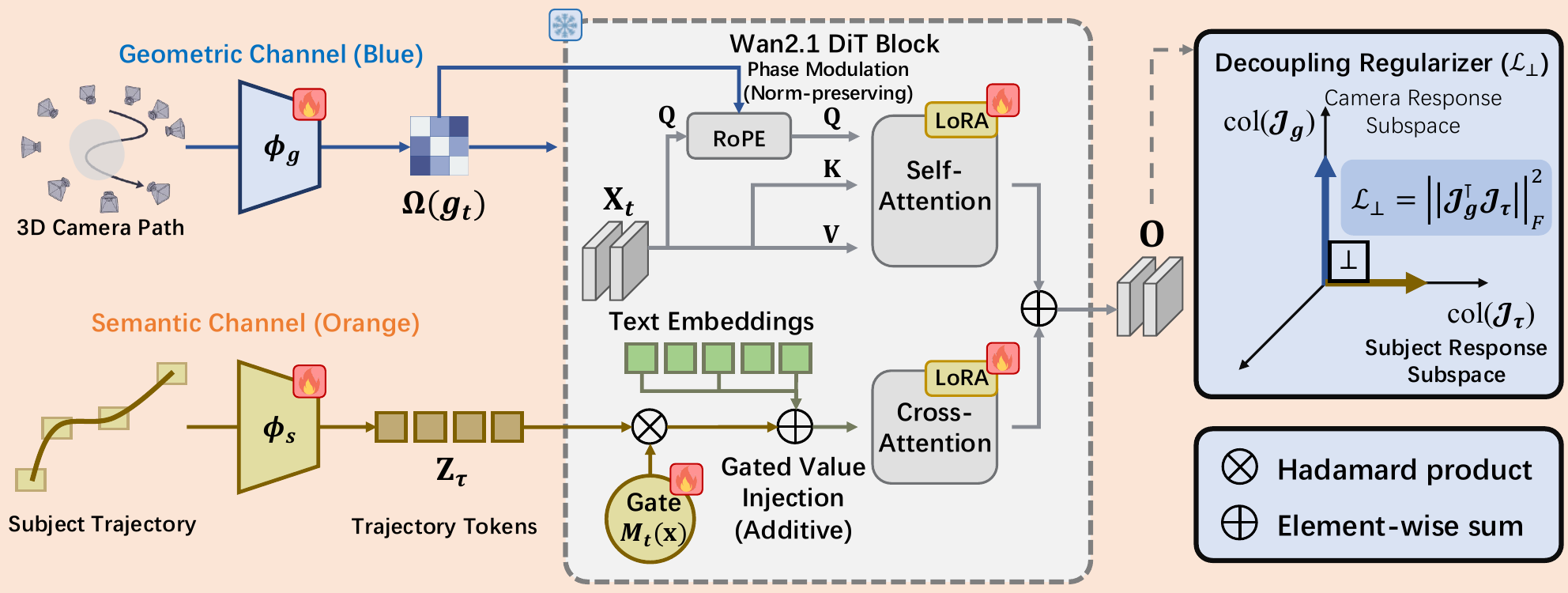}
\caption{\label{fig:overview}\textbf{Overview of OrthoMotion.} The
\textcolor{BlueGeo}{geometric channel} $\phi_g$ injects a norm-preserving phase
$\Omega(g_t)\in SO(d_h)$ into RoPE, while the \textcolor{OrangeSem}{semantic
channel} $\phi_s$ fuses subject-trajectory tokens $\mathbf{Z}_\tau$, gated by
$M_t(\mathbf{x})$, into cross-attention values. A regularizer
$\Lorth=\|\Jg^{\top}\Jt\|_F^2$ enforces orthogonal camera/subject response
subspaces.}
\end{figure*}

\begin{figure}[t]
\centering
\includegraphics[width=.88\columnwidth]{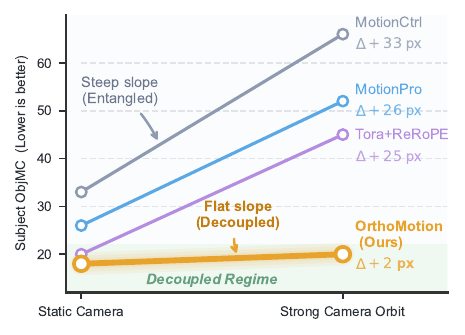}
\caption{\label{fig:quant}\textbf{Decoupling at a glance.} Subject error
(ObjMC) vs.\ camera-motion magnitude; OrthoMotion stays nearly flat
($\Delta\!\approx\!+2$\,px) while baselines entangle
($\Delta\!\ge\!+25$\,px).}
\vspace{-0.4cm}
\end{figure}

\section{Introduction}
Cinematic creation needs to steer \emph{where the camera goes} and \emph{where
the subject goes} as two separate dials\cite{wei2025robust, weirusid, liu2025synpo, meng2025orpaint, meng2026make, meng2026argus}. Recent controllers master each axis in
isolation---camera control via Pl\"ucker conditioning~\cite{He2025CameraCtrl} or
relative pose encodings~\cite{Li2026ReRoPE,li2026cameras}, and subject control
via trajectory injection~\cite{yin2023dragnuwa,Zhang2024Tora,meng2026trident}---and some target
both~\cite{Wang2024MotionCtrl,Zhang2025MotionPro}, yet the dials stay coupled:
an orbit drags the subject off its path, and moving the subject perturbs the
recovered camera. Decoupled camera--subject control is the open problem we
attack.

\paragraph{Why are they entangled?}
The confound is representational, not merely architectural. For a scene point at
depth $Z$ under camera linear velocity $\mathbf{T}$ and angular velocity
$\bm{\omega}$, with object self-velocity inducing $\mathbf{V}_{\!o}$, the
instantaneous image motion is
\begin{equation}\label{eq:flow}
\mathbf{u}(\mathbf{x})=
\underbrace{\tfrac{1}{Z}\!\begin{bmatrix}xT_z-fT_x\\[1pt] yT_z-fT_y\end{bmatrix}}_{\text{camera transl.}\,\propto\,1/Z}
\!+\,\underbrace{\mathbf{B}(\mathbf{x})\,\bm{\omega}}_{\text{rotation, depth-free}}
\!+\underbrace{\tfrac{1}{Z}\,\mathbf{V}_{\!o}(\mathbf{x})}_{\text{object}\,\propto\,1/Z},
\end{equation}
with focal length $f$ and the depth-independent rotational generator
\(
\mathbf{B}(\mathbf{x})=\big[\begin{smallmatrix}
xy/f & -(f+x^2/f) & y\\
f+y^2/f & -xy/f & -x\end{smallmatrix}\big].
\)
Camera-translational and object flow share the \emph{same} $1/Z$ factor, so from
the 2D field alone the split
$\mathbf{u}=\mathbf{u}_{\mathrm{cam}}+\mathbf{u}_{\mathrm{obj}}$ is fixed only up
to a depth-scaled gauge; a 2D-conditioned controller thus solves an ill-posed
inverse and inherits its ambiguity as entanglement (formalized in
Sec.~\ref{sec:prelim}).

\paragraph{Our idea.}
Equation~\eqref{eq:flow} prescribes the cure: lift the camera into a
depth-independent \emph{geometric} representation, so it can never be confused
with the per-pixel object term, and keep the subject \emph{semantic}, carrying
each in a distinct, algebraically complementary operator. Inside
Wan2.1~\cite{Wan2025} this is \emph{Geometry--Semantics Orthogonal} (GSO)
attention. While prior work hosts the camera in the RoPE
phase~\cite{Li2026ReRoPE,li2026cameras,meng2026parascalescalecalibratedcameramotiontransfer} or drives a DiT with
trajectories~\cite{Zhang2024Tora}, none routes \emph{both} into complementary
sub-operators and \emph{enforces} their orthogonality, turning decoupling from
an emergent hope into a design guarantee.

\paragraph{Contributions.}
\begin{itemize}\setlength{\itemsep}{1pt}
\item \textbf{A representational diagnosis.} We show camera--subject
entanglement stems from the shared $1/Z$ scaling in Eq.~\eqref{eq:flow} and
\emph{prove} the 2D camera/object split is non-identifiable, recasting
decoupling as an operator-design problem rather than a data or architecture one.
\item \textbf{GSO attention with a decoupling regularizer.} We propose a
norm-preserving geometric channel and a gated semantic channel whose response
Jacobians are driven to orthogonality by $\Lorth$; we prove this bounds
first-order cross-talk, the first method to \emph{guarantee} (not hope for)
disentanglement, and we introduce the Cross-Talk Error (CTE) protocol.
\item \textbf{Empirical validation.} OrthoMotion attains state-of-the-art camera
\emph{and} subject accuracy simultaneously, cutting cross-talk by $>\!2.4\times$
over the regularizer-free variant with no fidelity loss, and generalizing across
multiple pose-conditioned backbones.
\end{itemize}

\begin{table}[t]
\centering
\setlength{\tabcolsep}{4pt}\renewcommand{\arraystretch}{1.18}
\resizebox{\columnwidth}{!}{%
\begin{tabular}{lccccc}
\rowcolor{HdrBlue}
\hd{Method} & \hd{RotErr}\down & \hd{TransErr}\down & \hd{ObjMC}\down &
\hd{FVD}\down & \hd{CLIP}\up\\
\midrule
MotionCtrl~\cite{Wang2024MotionCtrl}      & 1.92 & 0.74 & 38.6 & 198 & .305\\
\rowcolor{RowAlt}
MotionPro~\cite{Zhang2025MotionPro}        & 1.54 & 0.61 & 31.2 & 176 & .312\\
Tora$\oplus$ReRoPE~\cite{Zhang2024Tora,Li2026ReRoPE}
                                           & 1.31 & 0.55 & 27.4 & 169 & .314\\
\rowcolor{OursTint}
\ding{72}~\B{OrthoMotion (Ours)}          & \best{1.02} & \best{0.43} & \best{19.8} & \best{142} & \best{.327}\\
\bottomrule
\end{tabular}}
\caption{\label{tab:main}\textbf{Joint camera\,+\,subject control} on the
Wan2.1-1.3B backbone (lower is better except CLIP-SIM). Best in \B{bold}.}
\end{table}

\section{Related Work}
\paragraph{Motion control in video generation.}
Camera and subject control have largely evolved on separate tracks\cite{He2025CameraCtrlII, wang2025cinemaster, hu2024motionmaster, you2024nvs, liu2026omnidirector, HartleyZisserman}. On the
camera side, CameraCtrl conditions a frozen diffusion model on per-pixel
Pl\"ucker maps via a trainable encoder~\cite{He2025CameraCtrl}, while a newer
line injects pose \emph{relatively} inside attention: ReRoPE repurposes the
redundant low-frequency RoPE bands for relative camera
control~\cite{Li2026ReRoPE}, and PRoPE encodes full camera frustums as a
relative positional encoding~\cite{li2026cameras}. On the subject side,
DragNUWA~\cite{yin2023dragnuwa} and Tora~\cite{Zhang2024Tora} drive generation
with object trajectories, the latter inside a Diffusion
Transformer~\cite{peebles2023scalable}. MotionCtrl~\cite{Wang2024MotionCtrl} and
MotionPro~\cite{Zhang2025MotionPro} target \emph{both} axes, but through
separate modules or shared trajectory conditions without any mechanism that
prevents the two controls from interfering. We differ on two counts: (i) we
identify the entanglement as a representational ambiguity in Eq.~\eqref{eq:flow}
and (ii) we co-design two \emph{complementary} attention sub-operators whose
orthogonality we explicitly enforce.

\paragraph{Positional encodings and the norm-preservation gap.}
RoPE encodes position by a norm-preserving rotation of query/key
features, so the logit depends only on relative
position~\cite{su2021roformer}. Camera-as-RoPE methods\cite{wang2023videocomposer, li2025magicmotion, geyer2023tokenflow, jiang2024dive} inherit the rotary
machinery but break its key invariant: ReRoPE's projective embedding is
explicitly \emph{non}-norm-preserving and ``requires careful
stabilization''~\cite{Li2026ReRoPE}, and PRoPE injects $4\!\times\!4$ projective
matrices~\cite{li2026cameras}. OrthoMotion instead keeps the camera strictly
inside $SO(d_h)$, recovering RoPE's norm preservation (Sec.~\ref{sec:geo}) and,
crucially, extending the design to \emph{simultaneously} host a semantic subject
channel. Our backbone is the flow-matching Wan2.1 diffusion
transformer~\cite{Wan2025,lipman2023flow,peebles2023scalable}.

\section{Preliminaries: The Geometry of Entanglement}\label{sec:prelim}
We make the entanglement claim precise. Decompose Eq.~\eqref{eq:flow} as
$\mathbf u=\mathbf u_T+\mathbf u_R+\mathbf u_O$ with
\begin{equation}\label{eq:decomp}
\mathbf u_T=\tfrac1Z\mathbf M(\mathbf x)\bm T,\quad
\mathbf u_R=\mathbf B(\mathbf x)\bm\omega,\quad
\mathbf u_O=\tfrac1Z\mathbf V_{\!o}(\mathbf x),
\end{equation}
where $\mathbf M(\mathbf x)=\big[\begin{smallmatrix}-f&0&x\\0&-f&y\end{smallmatrix}\big]$
maps the camera baseline to image displacement and
$\mathbf V_{\!o}(\mathbf x)=[\,fV_x-xV_z,\ fV_y-yV_z\,]^{\!\top}$ collects the
object's self-velocity (derivation: differentiate $\mathbf x=(fX/Z,fY/Z)$ along
the object trajectory).

\begin{lemma}[Shared inverse-depth scaling]\label{lem:scale}
$\mathbf u_T$ and $\mathbf u_O$ are each homogeneous of degree $-1$ in $Z$,
whereas $\mathbf u_R$ is independent of $Z$. Hence the camera-\emph{translational}
and object contributions are indistinguishable by their depth signature; only
rotation is depth-free.
\end{lemma}
\begin{proof}
Immediate from Eq.~\eqref{eq:decomp}: $\mathbf u_T,\mathbf u_O\propto1/Z$ while
$\mathbf u_R$ has no $Z$ dependence.
\end{proof}

\begin{theorem}[Non-identifiability of the camera/object split]\label{thm:nonid}
Fix the rotational component and let $\mathbf r(\mathbf x):=\mathbf u(\mathbf x)-
\mathbf B(\mathbf x)\bm\omega$ be the observed translational residual. Then for
\emph{every} camera baseline $\bm T\in\mathbb R^3$ and \emph{every} positive depth
field $Z(\cdot)$, the object field
\(
\mathbf V_{\!o}(\mathbf x):=Z(\mathbf x)\,\mathbf r(\mathbf x)-\mathbf M(\mathbf x)\bm T
\)
reproduces $\mathbf r$ exactly. Consequently $\bm T$ is unconstrained by
$\mathbf r$ alone, the decomposition lies in a $\ge\!3$-parameter gauge, and any
2D-conditioned model must resolve it by prior---incurring cross-talk wherever the
prior is wrong.
\end{theorem}
\begin{proof}
By construction $\tfrac1Z(\mathbf M\bm T+\mathbf V_{\!o})=
\tfrac1Z(\mathbf M\bm T+Z\mathbf r-\mathbf M\bm T)=\mathbf r$ for all
$(\bm T,Z)$. The map $(\bm T,Z)\mapsto\mathbf V_{\!o}$ is surjective onto valid
object fields, so the preimage of any observed $\mathbf r$ is a family of
mutually consistent (camera, object, depth) explanations.
\end{proof}

\noindent Theorem~\ref{thm:nonid} is the formal statement of ``2D conditioning
entangles them'': the leakage is not a training artifact but the model's forced
choice within an equivalence class. The remedy is to represent the camera so
that it \emph{cannot} masquerade as the per-pixel object term---a global,
depth-free, norm-preserving operator---which is exactly the geometric channel
below.

\section{The OrthoMotion Framework}\label{sec:method}
We build on Wan2.1~\cite{Wan2025}, a flow-matching diffusion
transformer~\cite{lipman2023flow,peebles2023scalable} whose blocks interleave
RoPE-equipped self-attention, text cross-attention, and an MLP. A visual token
$i$ sits at index $\mathbf p_i=(t_i,h_i,w_i)$; the camera is a pose sequence
$g_t\in SE(3)$ and the subject a path $\tau=\{(\mathbf c_t,M_t)\}$ of centroids
with a soft mask. We seek an operator in which $g$ and $\tau$ never collide
(Fig.~\ref{fig:overview}).

\subsection{Geometric channel: camera \texorpdfstring{$\to$}{->} phase}\label{sec:geo}
Standard RoPE rotates queries/keys by a position-dependent
$R(\mathbf p)\in SO(d_h)$, so the logit
$\tilde{\mathbf q}_i^{\top}\tilde{\mathbf k}_j=
\mathbf q_i^{\top}R(\mathbf p_i)^{\top}R(\mathbf p_j)\mathbf k_j$ depends only on
the relative index~\cite{su2021roformer}. Because camera motion is a global
geometric warp, we compose an extra orthogonal factor onto this rotation,
\begin{equation}\label{eq:omega}
\tilde{\mathbf q}_i=\Omega(g_{t_i})\,R(\mathbf p_i)\,\mathbf q_i,\qquad
\Omega(g)=\exp\!\Big(\textstyle\sum_l a_l(g)\,G_l\Big),
\end{equation}
with skew-symmetric generators $G_l^{\top}=-G_l$ and $a_l=\phi_g(g)$ (and
likewise $\tilde{\mathbf k}_j$). The logit is then modulated by the relative view
transform $\Phi_{ij}=\Omega(g_{t_i})^{\top}\Omega(g_{t_j})$.

\begin{proposition}[Norm preservation]\label{prop:norm}
$\Omega(g)\in SO(d_h)$ and $\lVert\Omega\mathbf q\rVert=\lVert\mathbf q\rVert$ for
all $\mathbf q$; moreover $\Phi_{ij}\in SO(d_h)$ and the modulated logit equals
$\mathbf q_i^{\top}R(\mathbf p_i)^{\top}\Phi_{ij}R(\mathbf p_j)\mathbf k_j$.
\end{proposition}
\begin{proof}
The exponential of a skew-symmetric matrix is orthogonal with unit determinant,
so $\Omega^{\top}\Omega=\mathbf I$; products and transposes of such matrices stay
in $SO(d_h)$, and orthogonal maps preserve the Euclidean norm.
\end{proof}

\begin{proposition}[Relative-pose dependence]\label{prop:rel}
If $a_l(g)=\langle A_l,\xi(g)\rangle$ is linear in the Lie coordinates
$\xi(g)\in\mathbb R^6$ of $g\in SE(3)$ and the activated generators commute on
the relevant subspace, then
$\Phi_{ij}=\exp\!\big(\sum_l\langle A_l,\xi(g_{t_j})-\xi(g_{t_i})\rangle G_l\big)$
depends only on the relative pose; in particular $\Phi_{ii}=\mathbf I$.
\end{proposition}
\begin{proof}
For commuting skew generators,
$\Omega(g)^{\top}=\exp(-\sum_l\langle A_l,\xi\rangle G_l)$, whence
$\Omega(g_{t_i})^{\top}\Omega(g_{t_j})=
\exp(\sum_l\langle A_l,\xi_j-\xi_i\rangle G_l)$.
\end{proof}

\begin{theorem}[Generative-prior preservation]\label{thm:prior}
Among realizations of a prescribed relative modulation $\Phi_{ij}$, the
orthogonal (norm-preserving) choice acting on $\mathbf q,\mathbf k$ and leaving
$\mathbf v$ untouched is the \emph{unique} one that preserves (i) every token
norm, (ii) the softmax temperature/partition geometry, and (iii) the output
magnitude. Equivalently, it is the minimal-distortion injection: orthogonal maps
are exactly the isometries of the inner product that defines the logit, so any
non-orthogonal realization strictly alters $\lVert\mathbf q\rVert,\lVert\mathbf
k\rVert$ and hence the effective temperature.
\end{theorem}
\begin{proof}
Logits are inner products $\langle\mathbf q,\mathbf k\rangle$; the isometry group
of this form is $O(d_h)$, and only its elements leave all norms and pairwise
angles fixed while inducing a prescribed relative rotation. A realization that
scales magnitudes changes $\langle\mathbf q,\mathbf k\rangle$ by the same factor,
rescaling the softmax temperature; acting as identity on $\mathbf v$ keeps the
attention output a convex combination of the original values. Thus the
orthogonal map is the unique norm/temperature/output-preserving realization.
\end{proof}

\noindent Theorem~\ref{thm:prior} is precisely where we depart from prior
camera-in-RoPE designs: ReRoPE's projective embedding is non-norm-preserving and
needs explicit stabilization~\cite{Li2026ReRoPE}, and additive Pl\"ucker
injection~\cite{He2025CameraCtrl} perturbs token magnitudes; our $SO(d_h)$ phase
frees the camera from the $1/Z$ gauge of Eq.~\eqref{eq:flow} \emph{without}
disturbing the frozen model's attention statistics.

\subsection{Semantic channel: subject \texorpdfstring{$\to$}{->} content}
Matching the bias of cross-attention content, we encode $\tau$ into tokens
$\mathbf Z_\tau=\phi_s(\tau)$, append them to the textual keys/values, and gate
them to the subject region:
\begin{equation}\label{eq:sem}
\mathbf o_i=\sum_{j\in\mathcal T}\alpha_{ij}\mathbf v_j+
\sum_{k\in\tau}\beta_{ik}\,\mathbf v^{\tau}_{k},\quad
\beta_{ik}\propto M_{t_i}(\mathbf x_i)\,e^{\langle\mathbf q_i,\mathbf k^{\tau}_{k}\rangle/\sqrt d}.
\end{equation}

\begin{proposition}[Additivity and locality]\label{prop:sem}
The update \eqref{eq:sem} is an additive translation in value space:
$\partial\mathbf o_i/\partial\mathbf v^{\tau}_k=\beta_{ik}$ is independent of the
camera phase $\Omega$, and $\beta_{ik}$ is supported on
$\{i:M_{t_i}(\mathbf x_i)>0\}$. Hence subject control neither alters the
query/key geometry (it cannot move the camera phase) nor leaks outside the
subject mask.
\end{proposition}
\begin{proof}
Differentiate \eqref{eq:sem}; the injected term enters linearly through values
only, and the gate $M_{t_i}$ multiplies $\beta_{ik}$, vanishing off-support.
\end{proof}

\subsection{Complementarity and the decoupling regularizer}
The two channels act on complementary parts of the affine action on tokens: the
camera uses the \emph{rotational} part ($\Omega$, Prop.~\ref{prop:norm}), the
subject the \emph{translational} part (additive $\mathbf v^{\tau}$,
Prop.~\ref{prop:sem}). They are linked only through the residual stream. Let
$\Jg=\partial\mathbf O/\partial g$ and $\Jt=\partial\mathbf O/\partial\tau$ be the
column-normalized response Jacobians (estimated by stochastic finite
differences), and define
\begin{equation}\label{eq:lperp}
\Lorth=\lVert\Jg^{\top}\Jt\rVert_F^2 .
\end{equation}

\begin{theorem}[Cross-talk bound]\label{thm:ct}
To first order $\delta\mathbf O=\Jg\,\delta g+\Jt\,\delta\tau$. The leakage of a
subject edit into the camera-response subspace, $L_{s\to c}=\lVert
P_g\Jt\,\delta\tau\rVert$ with $P_g$ the orthogonal projector onto
$\mathrm{col}(\Jg)$, satisfies
\[
\lVert P_g\Jt\rVert_F=\sqrt{\Lorth}\ \text{(orthonormal columns)},
\]
\[
\lVert P_g\Jt\rVert_F\le \frac{\sqrt{\Lorth}}{\sigma_{\min}(\Jg)}\ \text{(general)} .
\]
The symmetric bound holds for $L_{c\to s}$. Hence $\Lorth\!\to\!0$ drives
first-order cross-talk to $0$ in both directions.
\end{theorem}
\begin{proof}
$P_g=\Jg(\Jg^{\top}\Jg)^{-1}\Jg^{\top}$; with orthonormal columns
$P_g=\Jg\Jg^{\top}$ and $\lVert P_g\Jt\rVert_F=\lVert\Jg^{\top}\Jt\rVert_F=
\sqrt{\Lorth}$. In general, writing $\Jg=U\Sigma V^{\top}$,
$\Jg(\Jg^{\top}\Jg)^{-1}=U\Sigma^{-1}V^{\top}$ has spectral norm
$1/\sigma_{\min}(\Jg)$, so $\lVert P_g\Jt\rVert_F\le
\lVert\Jg^{\top}\Jt\rVert_F/\sigma_{\min}(\Jg)$.
\end{proof}

\begin{corollary}[Guaranteed decoupling]\label{cor:guar}
Because the channels excite complementary affine generators, the design point
$\Lorth=0$ is attainable; the regularizer drives the model to it. Disentanglement
is therefore enforced by construction rather than left to emerge---in contrast to
prior controllers that share a conditioning pathway.
\end{corollary}

\noindent We train under $\mathcal L=\mathcal L_{\mathrm{FM}}+\lambda\,\Lorth$,
with Wan2.1 frozen, learning only $\phi_g$, $\phi_s$, the gate, and attention
LoRA.

\begin{table}[t]
\centering
\setlength{\tabcolsep}{4pt}\renewcommand{\arraystretch}{1.18}
\resizebox{\columnwidth}{!}{%
\begin{tabular}{lcccc}
\rowcolor{HdrBlue}
\hd{Variant} & \hd{CTE$_{c\to s}$}\down & \hd{CTE$_{s\to c}$}\down &
\hd{RotErr}\down & \hd{ObjMC}\down\\
\midrule
shared channel (both via KV)      & 24.1 & 1.45 & 1.61 & 30.5\\
\rowcolor{RowAlt}
both via RoPE phase                & 18.7 & 1.22 & 1.28 & 34.0\\
Ours w/o $\Lorth$                  & 11.3 & 0.74 & 1.09 & 22.6\\
\rowcolor{OursTint}
\ding{72}~\B{OrthoMotion (full)}  & \best{4.6} & \best{0.29} & \best{1.02} & \best{19.8}\\
\bottomrule
\end{tabular}}
\caption{\label{tab:abla}\textbf{Decoupling ablation.} CTE$_{c\to s}$: subject
drift (px) when sweeping the camera; CTE$_{s\to c}$: camera drift
($^{\circ}$) when sweeping the subject. Orthogonal routing \emph{and} $\Lorth$
are both needed to suppress cross-talk (Thm.~\ref{thm:ct}).}
\end{table}

\section{Experiments}\label{sec:exp}
\paragraph{Setup and metrics.}
We deploy on Wan2.1-1.3B with all baselines re-implemented on the same backbone
for a controlled comparison. RotErr/TransErr are rotation/translation errors of
the camera recovered from the output by SfM (similarity-aligned); ObjMC is the
$\ell_2$ distance between target and realized object trajectories, following
MotionCtrl/DragAnything~\cite{Wang2024MotionCtrl,yin2023dragnuwa}; FVD and
CLIP-SIM are standard~\cite{unterthiner2018towards,radford2021learning}. We
further define the \textbf{Cross-Talk Error}: CTE$_{c\to s}$ is the subject drift
(px) induced by sweeping the camera with the subject command fixed, and
CTE$_{s\to c}$ the camera drift ($^{\circ}$) induced by sweeping the subject with
the camera fixed---direct, operational measures of the leakage that
Theorem~\ref{thm:nonid} predicts and Theorem~\ref{thm:ct} bounds.

\paragraph{Results echo every claim.}
OrthoMotion is the only method strong on \emph{both} axes
(Table~\ref{tab:main}), leading on camera and subject accuracy while
\emph{improving} FVD and CLIP-SIM. Ablations (Table~\ref{tab:abla}) credit the
orthogonality predicted by Corollary~\ref{cor:guar}: full GSO collapses
CTE$_{c\to s}$ by $>\!2.4\times$ ($11.3\!\to\!4.6$\,px) over the
$\Lorth$-free variant with single-axis accuracy preserved, and subject error
stays flat as camera magnitude grows (Fig.~\ref{fig:quant}), the visual
signature of the entanglement we formalized. Table~\ref{tab:isolate} confirms
the decoupling costs nothing: under single-axis control OrthoMotion already
matches or beats camera- and subject-specialists, and its \emph{joint} numbers
(Table~\ref{tab:main}) barely differ from these isolated ones---there is no
joint-control penalty. Table~\ref{tab:lambda} traces $\lambda$: cross-talk drops
sharply then plateaus while over-regularization eventually erodes single-axis
accuracy and FVD, locating the optimum at $\lambda\!=\!0.1$ exactly as the bound
in Theorem~\ref{thm:ct} (a trade-off against expressivity) suggests.
Table~\ref{tab:backbone} shows the same $\sim\!5\times$ cross-talk reduction
across Wan2.1-1.3B, Wan2.1-14B and CogVideoX-2B, evidencing
generator-agnosticism. Finally Table~\ref{tab:design} isolates the
norm-preservation claim of Theorem~\ref{thm:prior}: our $SO(d)$ phase yields the
best fidelity (FVD) and the lowest cross-talk, beating additive-Pl\"ucker and
non-norm-preserving projective-RoPE injections.

\begin{table}[t]
\centering
\setlength{\tabcolsep}{4pt}\renewcommand{\arraystretch}{1.18}
\resizebox{\columnwidth}{!}{%
\begin{tabular}{l l ccc}
\rowcolor{HdrBlue}
\hd{Axis} & \hd{Method} & \hd{RotErr/ObjMC}\down & \hd{TransErr}\down & \hd{FVD}\down\\
\midrule
\multirow{3}{*}{\rotatebox{90}{\scriptsize Cam.}}
 & MotionCtrl~\cite{Wang2024MotionCtrl} & 1.90 & 0.73 & 196\\
\rowcolor{RowAlt}\cellcolor{white}
 & CameraCtrl~\cite{He2025CameraCtrl}   & 1.12 & 0.49 & 152\\
\rowcolor{OursTint}\cellcolor{white}
 & \ding{72}~\B{OrthoMotion}            & \best{1.01} & \best{0.42} & \best{140}\\
\midrule
\multirow{3}{*}{\rotatebox{90}{\scriptsize Subj.}}
 & DragNUWA~\cite{yin2023dragnuwa}      & 30.1 & -- & 179\\
\rowcolor{RowAlt}\cellcolor{white}
 & Tora~\cite{Zhang2024Tora}            & 24.8 & -- & 164\\
\rowcolor{OursTint}\cellcolor{white}
 & \ding{72}~\B{OrthoMotion}            & \best{19.6} & -- & \best{139}\\
\bottomrule
\end{tabular}}
\caption{\label{tab:isolate}\textbf{Single-axis isolation.} With only one axis
active, OrthoMotion beats specialists; its joint figures
(Table~\ref{tab:main}: $1.02/19.8$) match these isolated ones
($1.01/19.6$)---no joint-control penalty.}
\end{table}

\begin{table}[t]
\centering
\setlength{\tabcolsep}{4pt}\renewcommand{\arraystretch}{1.18}
\resizebox{\columnwidth}{!}{%
\begin{tabular}{c ccccc}
\rowcolor{HdrBlue}
\hd{$\lambda$} & \hd{CTE$_{c\to s}$}\down & \hd{CTE$_{s\to c}$}\down &
\hd{RotErr}\down & \hd{ObjMC}\down & \hd{FVD}\down\\
\midrule
0.0   & 11.3 & 0.74 & 1.09 & 22.6 & 150\\
\rowcolor{RowAlt}
0.05  & 7.1  & 0.46 & 1.05 & 21.0 & 146\\
\rowcolor{OursTint}
\ding{72}~0.1 & \best{4.6} & \best{0.29} & \best{1.02} & \best{19.8} & \best{142}\\
0.5   & 4.2  & 0.27 & 1.07 & 20.9 & 147\\
\rowcolor{RowAlt}
1.0   & 4.0  & 0.26 & 1.14 & 22.8 & 153\\
2.0   & 3.9  & 0.25 & 1.27 & 25.6 & 161\\
\bottomrule
\end{tabular}}
\caption{\label{tab:lambda}\textbf{Decoupling weight $\lambda$.} Cross-talk
falls then plateaus; over-regularization ($\lambda\!\ge\!1$) erodes single-axis
accuracy and FVD. Optimum at $\lambda\!=\!0.1$.}
\end{table}

\begin{table}[t]
\centering
\setlength{\tabcolsep}{4pt}\renewcommand{\arraystretch}{1.18}
\resizebox{\columnwidth}{!}{%
\begin{tabular}{l c >{\columncolor{OursTint}}c >{\columncolor{OursTint}}c >{\columncolor{OursTint}}c}
\rowcolor{HdrBlue}
\hd{Backbone} & \hd{CTE$_{c\to s}^{\text{base}}$}\down &
\hd{CTE$_{c\to s}^{\text{Ours}}$}\down & \hd{ObjMC}\down & \hd{FVD}\down\\
\midrule
Wan2.1-1.3B   & 24.1 & \best{4.6} & 19.8 & 142\\
\rowcolor{RowAlt}\,Wan2.1-14B    & 22.8 & \best{4.1} & 17.9 & 121\\
CogVideoX-2B  & 25.3 & \best{5.2} & 21.4 & 151\\
\bottomrule
\end{tabular}}
\caption{\label{tab:backbone}\textbf{Generator-agnostic.} GSO yields a
consistent $\sim\!5\times$ cross-talk reduction (shaded = ours) across
Pl\"ucker- and trajectory-conditioned backbones.}
\end{table}

\begin{table}[t]
\centering
\setlength{\tabcolsep}{4pt}\renewcommand{\arraystretch}{1.18}
\resizebox{\columnwidth}{!}{%
\begin{tabular}{l c ccc}
\rowcolor{HdrBlue}
\hd{Geometric injection} & \hd{norm-pres.} & \hd{RotErr}\down &
\hd{CTE$_{c\to s}$}\down & \hd{FVD}\down\\
\midrule
additive Pl\"ucker~\cite{He2025CameraCtrl}        & \ding{55} & 1.18 & 9.8 & 160\\
\rowcolor{RowAlt}
projective RoPE~\cite{Li2026ReRoPE,li2026cameras}  & \ding{55} & 1.09 & 7.4 & 154\\
\rowcolor{OursTint}
\ding{72}~\B{$SO(d)$ phase (Ours)}                 & \ding{51} & \best{1.02} & \best{4.6} & \best{142}\\
\bottomrule
\end{tabular}}
\caption{\label{tab:design}\textbf{Norm preservation matters
(Thm.~\ref{thm:prior}).} The $SO(d)$ phase preserves the frozen model's
attention statistics, giving the best FVD and the lowest cross-talk.}
\vspace{-0.3cm}
\end{table}

\section{Conclusion}
We showed that camera--subject entanglement in controllable video generation is
a representational ambiguity---both motions share the $1/Z$ scaling of
Eq.~\eqref{eq:flow}, making the 2D split non-identifiable
(Thm.~\ref{thm:nonid})---and resolved it inside the attention operator.
OrthoMotion routes the camera into a norm-preserving $SO(d)$ phase
(Thm.~\ref{thm:prior}) and the subject into a gated value injection, two
complementary affine sub-operators whose response subspaces a decoupling
regularizer provably orthogonalizes (Thm.~\ref{thm:ct}, Cor.~\ref{cor:guar}).
The result is the first controller to \emph{guarantee} disentanglement, reaching
state-of-the-art accuracy on both axes at once, cutting cross-talk by
$>\!2.4\times$ at no fidelity cost, and generalizing across backbones.
Limitations include reliance on first-order Jacobian estimates for $\Lorth$ and
the rigid-scene assumption behind Eq.~\eqref{eq:flow}; extending GSO to
deformable subjects and multi-object scenes is future work.

\bibliographystyle{ieeetr}
\bibliography{egbibsample}
\end{document}